\documentclass[conference]{IEEEtran}
\IEEEoverridecommandlockouts

\usepackage{cite}
\usepackage{amsmath,amssymb,amsfonts}
\usepackage{algorithmic}
\usepackage{graphicx}
\usepackage{textcomp}
\usepackage{subfigure}
\usepackage{xcolor}
\usepackage{multirow}
\usepackage{multicol}
\usepackage{caption}
\usepackage{float}
\usepackage{algorithm}
\usepackage{algorithmic}
\usepackage{mathtools}  
\usepackage{cases}      
\usepackage{amsthm} 
\newtheorem{theorem}{Theorem}

\def\BibTeX{{\rm B\kern-.05em{\sc i\kern-.025em b}\kern-.08em
		T\kern-.1667em\lower.7ex\hbox{E}\kern-.125emX}}
\begin{document}
	
	\title{Hadamard-Riemannian Optimization \\for Margin-Variance Ensemble}
	
    \author{
    \IEEEauthorblockN{Zexu Jin}
    \IEEEauthorblockA{{\fontsize{9pt}{11pt}\selectfont\textit{School of Mathematics}} \\
    \textit{Nanjing University}\\
    Jiangsu, China \\
    zexujin@smail.nju.edu.cn}
    }
	
    \maketitle
    \begin{abstract}
Ensemble learning has been widely recognized as a pivotal technique for boosting predictive performance by combining multiple base models. Nevertheless, conventional margin-based ensemble methods predominantly focus on maximizing the expected margin while neglecting the critical role of margin variance, which inherently restricts the generalization capability of the model and heightens its vulnerability to overfitting, particularly in noisy or imbalanced datasets. Additionally, the conventional approach of optimizing ensemble weights within the probability simplex often introduces computational inefficiency and scalability challenges, complicating its application to large-scale problems. To tackle these limitations, this paper introduces a novel ensemble learning framework that explicitly incorporates margin variance into the loss function. Our method jointly optimizes the negative expected margin and its variance, leading to enhanced robustness and improved generalization performance. Moreover, by reparameterizing the ensemble weights onto the unit sphere, we substantially simplify the optimization process and improve computational efficiency. Extensive experiments conducted on multiple benchmark datasets demonstrate that the proposed approach consistently outperforms traditional margin-based ensemble techniques, underscoring its effectiveness and practical utility.
\end{abstract}
    \begin{IEEEkeywords}
Ensemble Learning, Riemannian Optimization, Margin Variance Regularization, Hadamard Parameterization, Simplex Constraints.
\end{IEEEkeywords}
    \section{Introduction}
Ensemble learning\cite{dong_survey_2020},\cite{Polikar2012}, has established itself as a cornerstone of modern machine learning by significantly improving predictive performance through the strategic combination of multiple base models. Contemporary state-of-the-art approaches can be broadly categorized into homogenous and heterogeneous ensembles\cite{bian_diversity_2007}. Homogenous methods, such as Random Forests\cite{breiman_random_2001, Cutler2012, SPEISER201993} and gradient boosting machines\cite{bentejac_comparative_2021},\cite{pmlr-v108-lu20a}, construct a strong learner by aggregating a large number of weak learners of the same type. Conversely, heterogeneous ensembles leverage the complementary strengths of diverse model architectures to achieve superior performance\cite{10888549}. The theoretical underpinnings of these methods are often explained by their ability to reduce variance\cite{9226504}, optimize bias, and exploit model complementarity, thereby enhancing the overall generalization capability on unseen data.

Among the various theoretical frameworks for understanding ensembles, margin-based analysis has attracted significant attention\cite{yuan2024achieving,huang_large_2021,guenov_margin_2018}. This theory posits that the generalization error of a combined classifier is bounded by its margin distribution, which measures the confidence of classification predictions. Consequently, a substantial body of research has been devoted to designing ensemble methods that explicitly maximize the prediction margin\cite{7299663, 7965869, 7324382}.

However, despite their theoretical appeal, conventional margin-based ensemble methods suffer from two major and interconnected limitations. First, an overwhelming majority of these techniques myopically focus on maximizing the \textit{expected} or \textit{average} margin across the training set\cite{10888549}. This singular focus often comes at the cost of ignoring the \textit{variance} of the margin distribution. This neglect renders the model susceptible to overfitting, particularly on noisy data and outliers\cite{ying_overview_2019}, and leads to unstable predictions on unseen examples, ultimately compromising generalization. Second, from an optimization perspective, existing approaches typically constrain the ensemble weights to the probability simplex. Although intuitive, this constraint requires computationally expensive projection operations after each gradient step during optimization\cite{wang_projection_2013,7814152}. This requirement not only hinders training efficiency but also severely limits the scalability of these methods to large-scale datasets and complex model architectures\cite{9165233}.

To address these critical issues in a unified way, we propose a robust margin-aware ensemble learning framework with the following key innovations:

\begin{itemize}
\item \textbf{Margin Distribution-Optimized Loss:} We design a novel loss function that jointly optimizes both the expected margin and its variance. This promotes large and consistent margins across all samples, significantly improving robustness against adversarial examples.

\item \textbf{Sphere-Based Reparameterization:} We reparameterize the ensemble weights onto a unit sphere. This eliminates the computational overhead of simplex projections, substantially improving optimization efficiency.

\item \textbf{Riemannian Gradient Optimization:} We introduce Riemannian optimization methods to ensemble learning, which approach fully exploits the geometric characteristics of the spherical constraint.
\end{itemize}

Extensive experimental results demonstrate that our proposed algorithm exhibits significant advantages in generalization performance, training efficiency, and robustness, thus conclusively validating its effectiveness.






    \section{Method}
\label{sec:method}

Consider a $c$-class classification problem with $m$ base classifiers $g_1, \dots, g_m$ and $n$ instances $x_1, \dots, x_n$, each with a one-hot encoded label vector $Y_i \in \mathbb{R}^c$. The complete label matrix is $Y \in \mathbb{R}^{c \times n}$.

Each classifier is assigned a weight $w_i \geq 0$ with $\sum_{i=1}^m w_i = 1$. Let $W = [w_1, \dots, w_m]^\top \in \mathbb{R}^m$ be the weight vector. For instance $x_i$, let $G_i = [g_1(x_i), \dots, g_m(x_i)] \in \mathbb{R}^{c \times m}$ denote the prediction matrix from all classifiers.

The ensemble prediction is:
\begin{equation}
y_i = \mathop{\arg\max}\limits_{j = 1, \dots, c} (G_i W).
\end{equation}

We employ a margin-based approach for weight optimization. The margin for instance $x_i$ is defined as:
\begin{equation}
\text{margin}_i = (G_i W)^\top Y_i - \max\left(G_i W - Y_i \odot (G_i W)\right)
\end{equation}
where $\odot$ denotes element-wise product. A larger margin indicates greater classification confidence, while a negative margin indicates misclassification\cite{NEURIPS2020_c6102b37}.

To address non-smoothness, we approximate the max function using log-sum-exp\cite{9264376}:
{\small
\begin{equation}
\text{log-sum-exp}(f)=\frac{1}{\alpha} \log \sum_{j=1}^c \exp (\alpha f)
\end{equation}
}
where $\alpha > 0$ is a temperature parameter. This yields the smoothed margin\cite{wang2025cooperation}, $m_i=\text{margin}_i$:
{\small
\begin{equation}
m_i = (G_i W)^\top Y_i 
 - \frac{1}{\alpha} \log \left( \sum_{j=1}^{c} \exp\left(\alpha \left(G_i W - Y_i \odot (G_i W)\right)_j\right) \right)
\end{equation}
}

To improve generalization and prevent overfitting, we incorporate the variance of the margin into the loss function\cite{5586027} with a regularization parameter $\lambda \geq 0$. The complete loss function is formulated as:
\begin{equation}
\mathcal{L}(W) = -\frac{1}{n} \sum_{i=1}^n m_i + \lambda \left( \frac{1}{n} \sum_{i=1}^n m_i^2 - \left( \frac{1}{n} \sum_{i=1}^n m_i \right)^2 \right)
\end{equation}
This objective function simultaneously maximizes the average margin while controlling its variability, leading to improved robustness and generalization performance\cite{kanno_three_2020}.

    \section{OPTIMIZATION ALGORITHM}

This section details the employed optimization methods. After obtaining the loss function, we need to optimize the weights using a gradient projection method. Due to the $O(n \log n)$ time complexity of projecting onto the unit simplex, which makes projection computations computationally expensive,we introduce the Hadamard product parameterization, i.e., let \( w = z \odot z \), which transforms the simplex constraint into a spherical constraint\cite{li_simplex_2023}. 

We now prove the equivalence\cite{chen_kkt_2022},\cite{bergmann_intrinsic_2019} between these two optimization formulations:

\begin{theorem}[Simplex-Sphere Equivalence]
    Let \( f: \mathbb{R}^m \to \mathbb{R} \in C^2\). The problems
    \begin{equation}
    \min_{w \ge 0,\ \sum_i w_i = 1} f(w) \quad \text{and} \quad \min_{\|z\|_2 = 1} f(z \odot z)
    \end{equation}
	with \( w = z \odot z = (z_1^2, \dots, z_m^2) \) have equivalent first- and second-order KKT conditions.
\end{theorem}

\begin{proof}
\textbf{First-Order KKT Conditions}

    The first-order KKT conditions of the simplex problem are:
    \begin{small}
	\begin{align}
		&\nabla f(w^*) + \lambda^* \mathbf{1} - \mu^* = 0, \\
		&\sum_{i=1}^m w_i^* = 1,\ w_i^* \ge 0, \quad \mu_i^* \ge 0,\ \mu_i^* w_i^* = 0.
	\end{align}
    \end{small}
    
    The first-order KKT conditions of the sphere problem are:
	\begin{equation}
		2 \operatorname{diag}(z^*) \nabla f(w^*) + 2\gamma^* z^* = 0,\quad \sum_{i=1}^m (z_i^*)^2 = 1.
	\end{equation}
	
	Setting \( \gamma^* = \lambda^* \) and \( w^* = z^* \odot z^* \), the conditions correspond exactly:
	
	\begin{enumerate}
		\item When \( w_i^* > 0 \): \( \nabla f(w^*)_i + \lambda^* = 0 \)
		\item When \( w_i^* = 0 \): The sphere problem imposes no additional condition
        \vspace{-\baselineskip}
	\end{enumerate}
\end{proof}

\begin{proof}
	\textbf{Second-Order KKT Conditions}
	
	The second-order term for the sphere problem:
	\begin{equation}
		d^\top \left( \nabla^2 g(z^*) + 2\gamma^* I \right) d
	\end{equation}
	equals the second-order term for the simplex problem:
	\begin{equation}
		\delta w^\top \nabla^2 f(w^*) \delta w
	\end{equation}
	where \( \delta w = 2 \operatorname{diag}(z^*) d \) is a feasible direction.This establishes the equivalence of the second-order conditions.
\end{proof}

Subsequently, we apply Riemannian gradient descent\cite{smith_optimization_2014,yuan2025riemannian} on the unit sphere. This algorithm operates with $O(n)$ time complexity, making it computationally more efficient. The Riemannian gradient is obtained by projecting the Euclidean gradient onto the tangent space:
\begin{small}
\begin{equation}
proj,h(z) = P_z(\nabla_z h(z)) = \nabla_z h(z) - z(z^\top \nabla_z h(z))
\end{equation}
\end{small}
where $P_z = I - zz^\top$ projects from $\mathbb{R}^m$ to $T_z\mathbb{S}^{m-1}$, and $\nabla_z h(z) = 2z \odot \nabla f(z \odot z)$ by the chain rule.

Retraction maps the updated point back to the sphere while maintaining optimality conditions\cite{absil_projection-like_2012}.

The general steps of the Riemannian gradient descent on the unit sphere are as follows:

\begin{algorithm}[H]
	\caption{Riemannian Gradient Descent on the Unit Sphere}
	\label{alg:riemannian-gd}
	\begin{algorithmic}[1]
		\REQUIRE Initialize $z^{(0)} \in \mathbb{S}^{m-1}$, learning rate sequence $\{\alpha_t\}$
		\ENSURE Optimized point $z^{(T)}$
		\STATE Initialize $z_0$ randomly on the unit sphere
		\FOR{$t = 0, 1, 2, \ldots$ until convergence}
		\STATE Compute $w^{(t)} = z^{(t)} \odot z^{(t)}$
		\STATE Compute Euclidean gradient: $g_w = \nabla f(w^{(t)})$
		\STATE Compute reparameterized gradient: $g_z = 2z^{(t)} \odot g_w$
		\STATE Project onto tangent space: $\xi = g_z - z^{(t)} (z^{(t)^\top} g_z)$
		\STATE Update in tangent space: $y = z^{(t)} - \alpha_t \xi$
		\STATE Perform retraction: $z^{(t+1)} = y / \|y\|_2$
		\ENDFOR
	\end{algorithmic}
\end{algorithm}

Finally, we demonstrate that the proposed loss function exhibits desirable mathematical properties. The function is \textit{Lipschitz continuous}~\cite{pmlr-v70-malherbe17a,vinod_constrained_2022}, which guarantees numerical stability and facilitates convergence via gradient-based methods. Moreover, its \textit{convexity} with respect to the margin~\cite{5457713} ensures that in the margin space, there exists only a single global minimum without any local minima. This creates a favorable optimization landscape: as long as the optimization algorithm can effectively manipulate the margin values through changes in the parameters, it will not be trapped in spurious local optima. This property, combined with its Lipschitz continuity, makes the loss function well-suited for gradient-based optimization. Formal proofs of these properties are provided below.


\begin{theorem}[Lipschitz Continuity of Loss Function]
Assuming there exist $M_g > 0$ such that $|g_k(x_i)|_2 \leq M_g$ for all $i,k$ and $M_m > 0$ such that $|m_i| \leq M_m$ for all $i$, then the loss function $\mathcal{L}(W)$ is Lipschitz continuous: \begin{equation}
    L = M_g (1 + \sqrt{c}) (1 + 4\lambda M_m).
    \end{equation}
    \end{theorem}

\begin{proof}
	We prove the Lipschitz continuity by bounding the gradient norm.
	
	Let $\mathbf{s}_i = G_i W$. The smoothed margin is:
	\begin{equation}
		m_i = \mathbf{s}_i^\top Y_i - \frac{1}{\alpha}logsumexp(\alpha(\mathbf{s}_i - Y_i \odot \mathbf{s}_i)_j)
	\end{equation}
	
	Define $p_j = \frac{\exp(\alpha(\mathbf{s}_i - Y_i \odot \mathbf{s}_i)_j)}{\sum_{k=1}^c \exp(\alpha(\mathbf{s}_i - Y_i \odot \mathbf{s}_i)_k)}$, then since we have:
	
	\begin{equation}
		\|\nabla_{\mathbf{s}_i}(\mathbf{s}_i - Y_i \odot \mathbf{s}_i)_j\|_2 \leq 1 + \|Y_i\|_2 = 2
	\end{equation}
	\vspace{-\baselineskip} 
	\begin{small}
		\begin{equation}
			\|\nabla_{\mathbf{s}_i} m_i\|_2 = \|Y_i - \sum_{j=1}^c p_j \nabla_{\mathbf{s}_i}(\mathbf{s}_i - Y_i \odot \mathbf{s}_i)_j\|_2
			\leq 1 + 2 = 3
		\end{equation}
	\end{small}

	By the chain rule:
	\begin{small}
	\begin{equation}
		\|\nabla_W m_i\|_2 =\|G_i^\top \nabla_{\mathbf{s}_i} m_i\|_2 \leq \|G_i^\top\|_2 \cdot \|\nabla_{\mathbf{s}_i} m_i\|_2 \leq 3 M_g
	\end{equation}
	\end{small}
	For the variance term:
	\begin{equation}
		\nabla_W \mathrm{Var}[m] = \frac{2}{n}\sum_{i=1}^n (m_i - \bar{m})\nabla_W m_i
	\end{equation}
    
	\vspace{-\baselineskip}
    
	\begin{equation}
		\|\nabla_W \mathrm{Var}[m]\|_2 \leq 4M_m \cdot 3M_g = 12M_m M_g
	\end{equation}
	
	Combining both components:
	\begin{small}
	\begin{equation}
		\|\nabla_W \mathcal{L}(W)\|_2 \leq 3M_g + 12\lambda M_m M_g = 3M_g(1 + 4\lambda M_m)
	\end{equation}
	\end{small}
	
	By the mean value theorem, $\mathcal{L}(W)$ is Lipschitz continuous with constant $L = 3M_g(1 + 4\lambda M_m)$.
\end{proof}

\begin{theorem}[Convexity with Respect to Margin]
The loss function $\mathcal{L}$ is convex with respect to the margin vector $\mathbf{m} = (m_1, \dots, m_n)^\top$, where each $m_i$ represents the margin of the $i$-th instance.
\end{theorem}

\begin{proof}
	The loss function consists of two parts: a linear term $-\frac{1}{n}\sum_{i=1}^n m_i$ and a variance term $\lambda \cdot \mathrm{Var}(\mathbf{m})$. The linear term is an affine function, which is both convex and concave.
	
	For the variance term $\mathrm{Var}(\mathbf{m}) = \frac{1}{n}\sum_{i=1}^n (m_i - \bar{m})^2$, where $\bar{m} = \frac{1}{n}\sum_{i=1}^n m_i$, we compute its Hessian matrix. The second-order partial derivatives are:
    \begin{small}
        \begin{equation}
		\frac{\partial^2 \mathrm{Var}(\mathbf{m})}{\partial m_k \partial m_l} = \frac{2}{n}\left(\delta_{kl} - \frac{1}{n}\right)
        \end{equation}  
    \end{small}

	Thus, the Hessian matrix $\mathbf{H}$ can be expressed as:
    \begin{small}
    \begin{equation}
		\mathbf{H} = \frac{2}{n}\left(\mathbf{I} - \frac{1}{n}\mathbf{1}\mathbf{1}^T\right)
    \end{equation}    
    \end{small}

	where $\mathbf{I}$ is the $n \times n$ identity matrix and $\mathbf{1}$ is the all-ones vector. For any non-zero vector $\mathbf{v} \in \mathbb{R}^n$:
    \begin{small}
   	\begin{equation}
		\mathbf{v}^T \mathbf{H} \mathbf{v} = \frac{2}{n} \left(\|\mathbf{v}\|^2 - \frac{1}{n}(\mathbf{1}^T \mathbf{v})^2\right) \geq 0
	\end{equation}     
    \end{small}
	by the Cauchy-Schwarz inequality. Therefore, the Hessian matrix is positive semidefinite, and the variance term is convex.So the overall loss function is convex.
\end{proof}

    \section{EXPERIMENT}
\label{sec:experiment}

\begin{figure}[H]
	\centering
        \setlength{\tabcolsep}{2pt} 
	\begin{tabular}{ccc}
		\subfigure[Data distribution]{
			\includegraphics[width=0.13\textwidth]{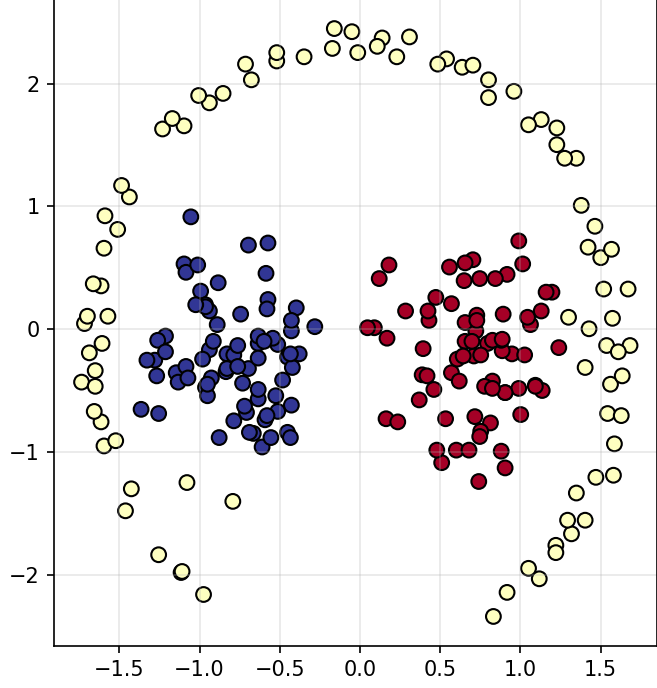} 
		} &
		\subfigure[Our method]{
			\includegraphics[width=0.13\textwidth]{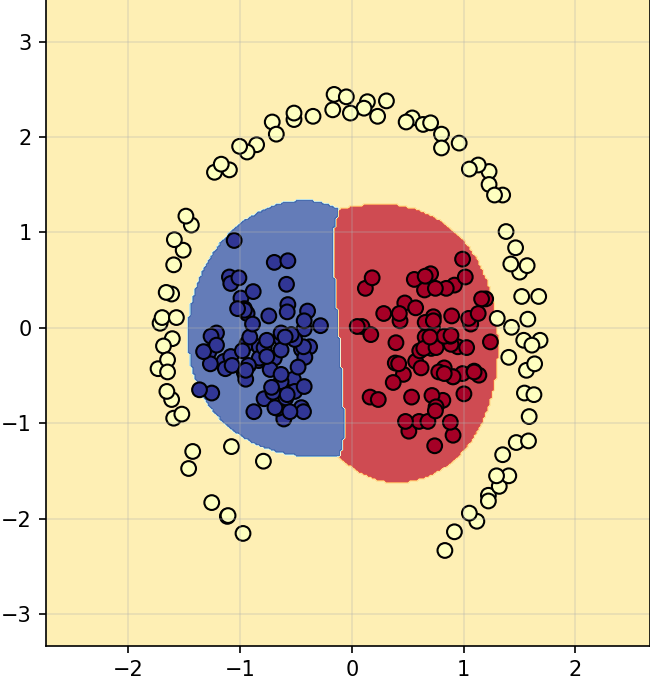} 
		} &
		\subfigure[RF100]{
			\includegraphics[width=0.13\textwidth]{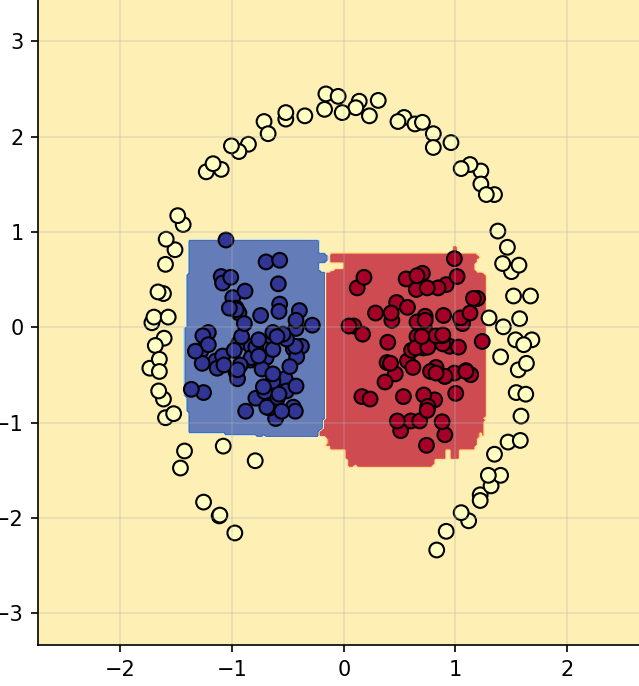} 
		} \\
		\subfigure[SVM]{
			\includegraphics[width=0.13\textwidth]{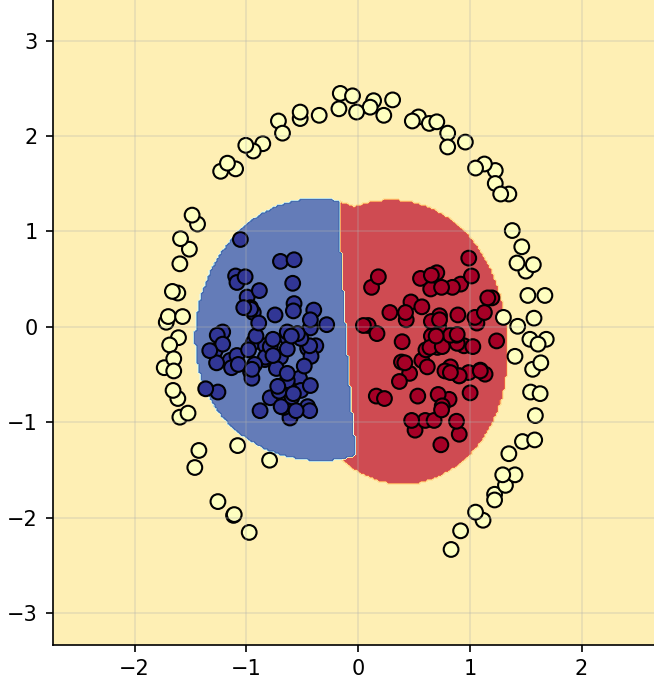} 
		} &
		\subfigure[XGBoost]{
			\includegraphics[width=0.13\textwidth]{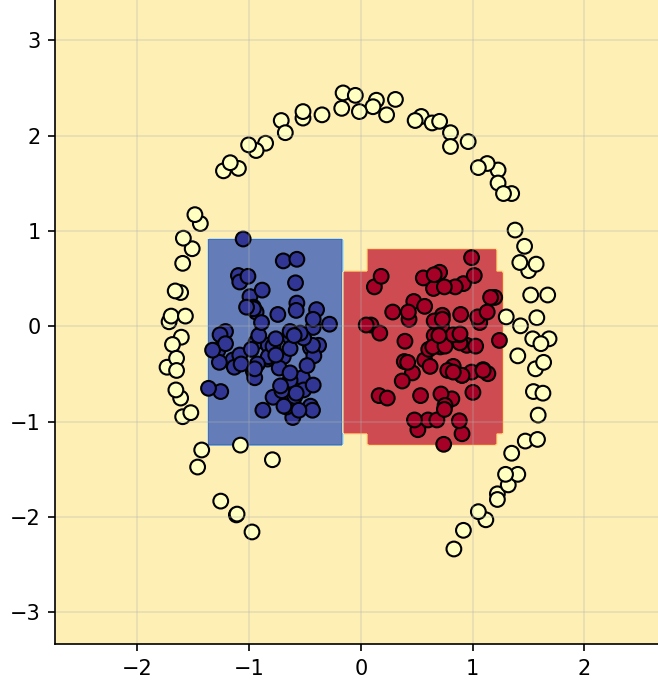}
		} &
		\subfigure[LightGBM]{
			\includegraphics[width=0.13\textwidth]{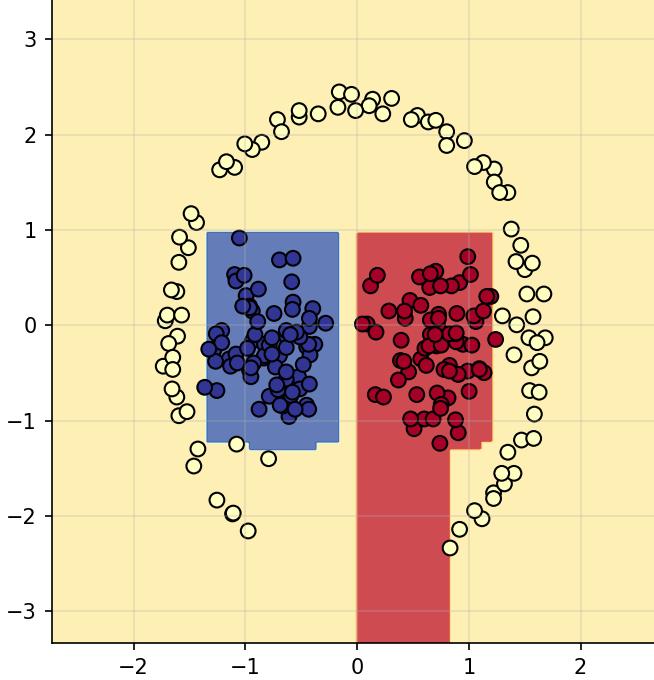} 
		}
	\end{tabular}
	\caption{Decision boundaries on Pathbased dataset}
	\label{F6}
\end{figure}

\subsection{Experimental Setup}
We evaluate our method on eight diverse datasets\cite{1190577, wang_generalized_2023, 10.1007/978-3-030-98015-3_17, KHAN201894} with varying sizes and dimensions. All models use a maximum depth of 7 for fair comparison. An 80:20 train-test split is applied, and the random seed is fixed as 0.

\begin{table*}
	\caption{Accuracy comparison on various datasets}
	\label{table:results}
	\centering
	\small
	\setlength{\tabcolsep}{4pt}
	\begin{tabular}{||c||c||ccccccccc||cc||}
		\hline
		Dataset & Split & Our-Best$\lambda$ & Our-$\lambda$=0 & RF50 & RF100 & SVM & XGBoost & LightGBM & Samples & Features \\
		\hline
		Basehock & Train & \underline{98.68} ($\lambda$=500) & 98.87 & 95.80 & 96.49 & \textbf{99.62} & 93.41 & 91.28 & 1594 & \multirow{2}{*}{4862} \\
		& Test  & \textbf{95.99} & 93.48 & 93.73 & \underline{95.74} & 94.49 & 90.48 & 88.97 & 399 &  \\
		\hline
		Breast\_uni & Train & 97.50 ($\lambda$=0.1) & 97.32 & 98.57 & \underline{98.57} & 97.50 & \textbf{99.11} & 97.14 & 559 & \multirow{2}{*}{10} \\
		& Test  & \underline{96.43} & 96.43 & \textbf{97.14} & 97.14 & 96.43 & 95.71 & 95.00 & 140 &  \\
		\hline
		Chess & Train & \textbf{98.75} ($\lambda$=40) & 98.63 & 94.25 & 94.13 & 98.47 & \underline{97.26} & 94.48 & 2557 & \multirow{2}{*}{36} \\
		& Test  & \textbf{98.28} & 97.81 & 94.69 & 94.53 & \underline{97.97} & 96.88 & 94.84 & 639 &  \\
		\hline
		Mnist-10000 & Train & \textbf{97.95} ($\lambda$=20) & 97.94 & 86.61 & 87.03 & \underline{98.38} & 96.93 & 95.17 & 10000 & \multirow{2}{*}{10} \\
		& Test  & \textbf{94.82} & 94.74 & 84.91 & 85.15 & \underline{93.78} & 91.26 & 90.50 & 10000 &  \\
		\hline
		Jaffe & Train & \textbf{100} ($\lambda$=0.01) & 100 & 100 & 100 & 100 & 100 & 100 & 170 & \multirow{2}{*}{1024} \\
		& Test  & \textbf{100} & 100 & \underline{97.67} & 97.67 & 100 & 100 & 97.67 & 43 &  \\
		\hline
		Pathbased & Train & \textbf{100} ($\lambda$=1.0) & 99.58 & 100 & 98.33 & \underline{99.58} & 100 & 99.17 & 240 & \multirow{2}{*}{2} \\
		& Test  & \textbf{100}  & 100 & 96.67 & 96.67 & \underline{98.33} & 98.33 & 100 & 60 &  \\
		\hline
		Relathe & Train & \underline{96.76} ($\lambda$=50) & 97.11 & 83.17 & 85.63 & \textbf{98.60} & 88.96 & 85.01 & 1142 & \multirow{2}{*}{4322} \\
		& Test  & \textbf{89.16} & 88.46 & 78.32 & 79.72 & \underline{88.81} & 80.07 & 79.02 & 285 &  \\
		\hline
		Wine & Train & \textbf{100} ($\lambda$=0.1) & 100 & 100 & 100 & \underline{99.30} & 100 & 100 & 142 & \multirow{2}{*}{13} \\
		& Test  & \textbf{100} & 100 & 100 & 100 & \underline{97.22} & 97.22 & 100 & 36 &  \\
		\hline
	\end{tabular}
\end{table*}

\subsection{Results and Analysis}
As shown in Table~\ref{table:results}, our ensemble achieves state-of-the-art performance, obtaining the highest test accuracy on most databases. The method excels on high-dimensional datasets like BASEHOCK (4,862 features) and RELATHE (4,322 features), demonstrating strong capability in complex feature spaces.The model shows excellent generalization with minimal overfitting. For example, on Chess, it maintains only a 0.47\% accuracy drop from training to test, outperforming methods like SVM which shows a 5.13\% drop on BASEHOCK.

\begin{table}[htbp]
	\centering
	\caption{Runtime(seconds) comparison of Hadamard Riemannian gradient method (Method 1) and direct projected gradient method (Method 2)}
	\label{table:runtime}
	\small
	\setlength{\tabcolsep}{3pt}
	\begin{tabular}{||c||cc|c||cc||}
		\hline
		Dataset & Method1 (s) & Method2 (s) & Speedup (\%) \\
		\hline
		BASEHOCK & 0.2458 & 0.3762 & 34.7  \\
		Breast\_uni & 0.1769 & 0.2241 & 21.1  \\
		Chess & 0.3843 & 0.5689 & 32.5  \\
		MNIST-10000 & 0.7541 & 1.8363 & 58.9  \\
		Jaffe & 0.1069 & 0.1347 & 20.6  \\
		Pathbased & 0.1555 & 0.1747 & 11.0  \\
		RELATHE & 0.2472 & 0.5587 & 55.8  \\
		Wine & 0.1493 & 0.1648 & 9.4  \\
		\hline
	\end{tabular}
\end{table}

Table~\ref{table:runtime} shows that our Hadamard-Riemannian gradient method achieves faster runtime across all datasets, with significant improvements on high-dimensional datasets (59\% on MNIST and 56\% on RELATHE), demonstrating practical advantages for large-scale problems.

\begin{figure}[h]
	\centering
        \setlength{\tabcolsep}{2pt} 
        \begin{tabular}{c@{\hspace{0.3cm}}c@{\hspace{0.3cm}}c}
		\subfigure[BASEHOCK]{
			\includegraphics[width=0.1325\textwidth]{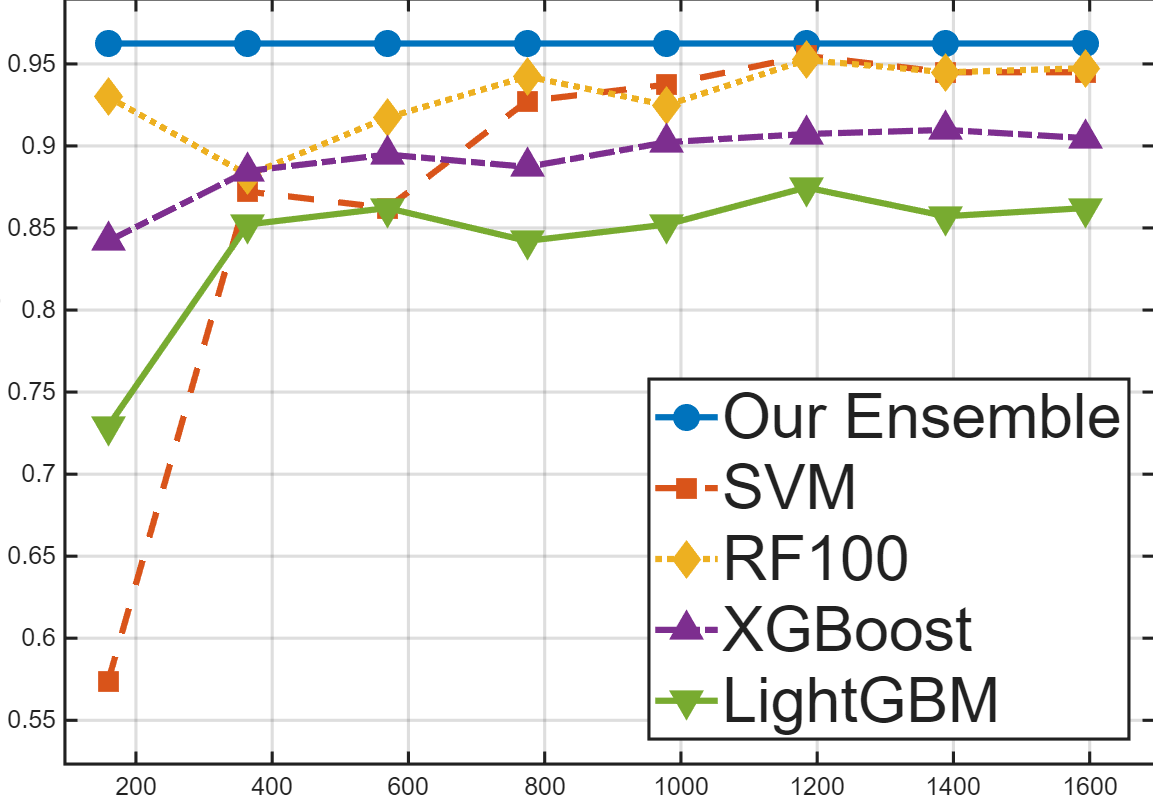}
		} &
		\subfigure[Chess]{
			\includegraphics[width=0.1325\textwidth]{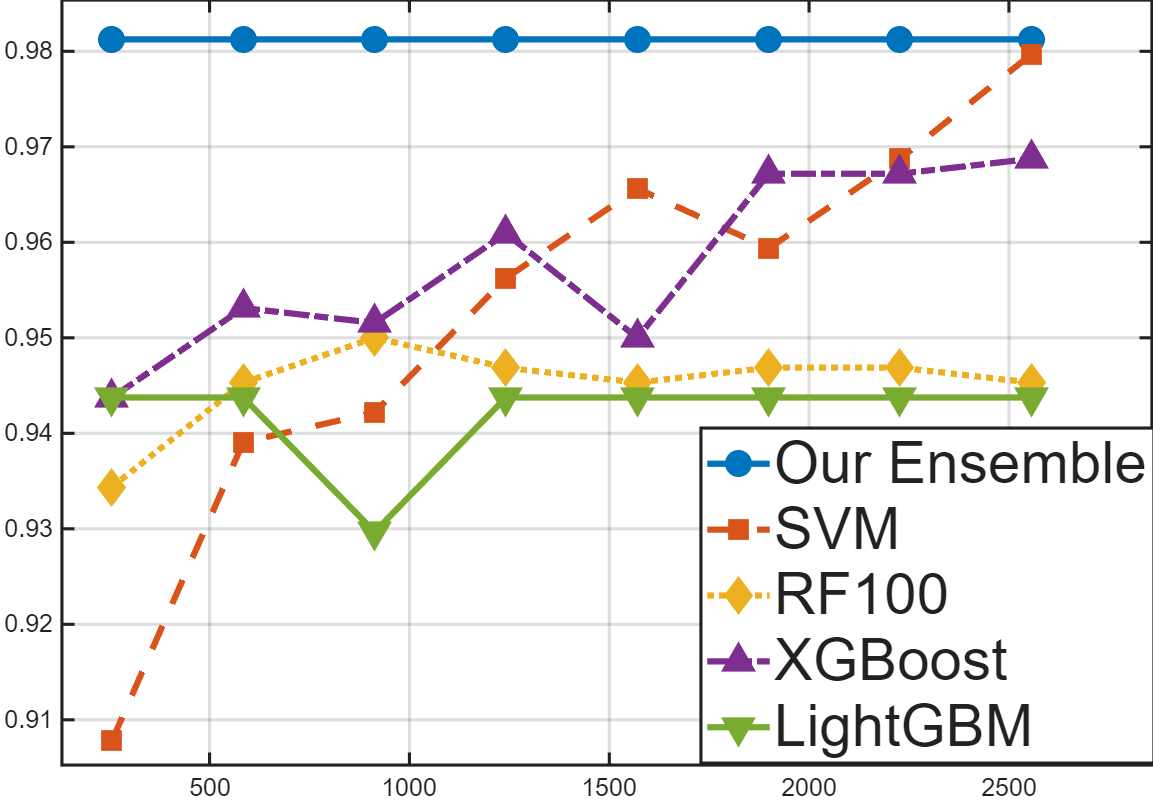}
		} &
		\subfigure[Mnist-10000]{
			\includegraphics[width=0.1325\textwidth]{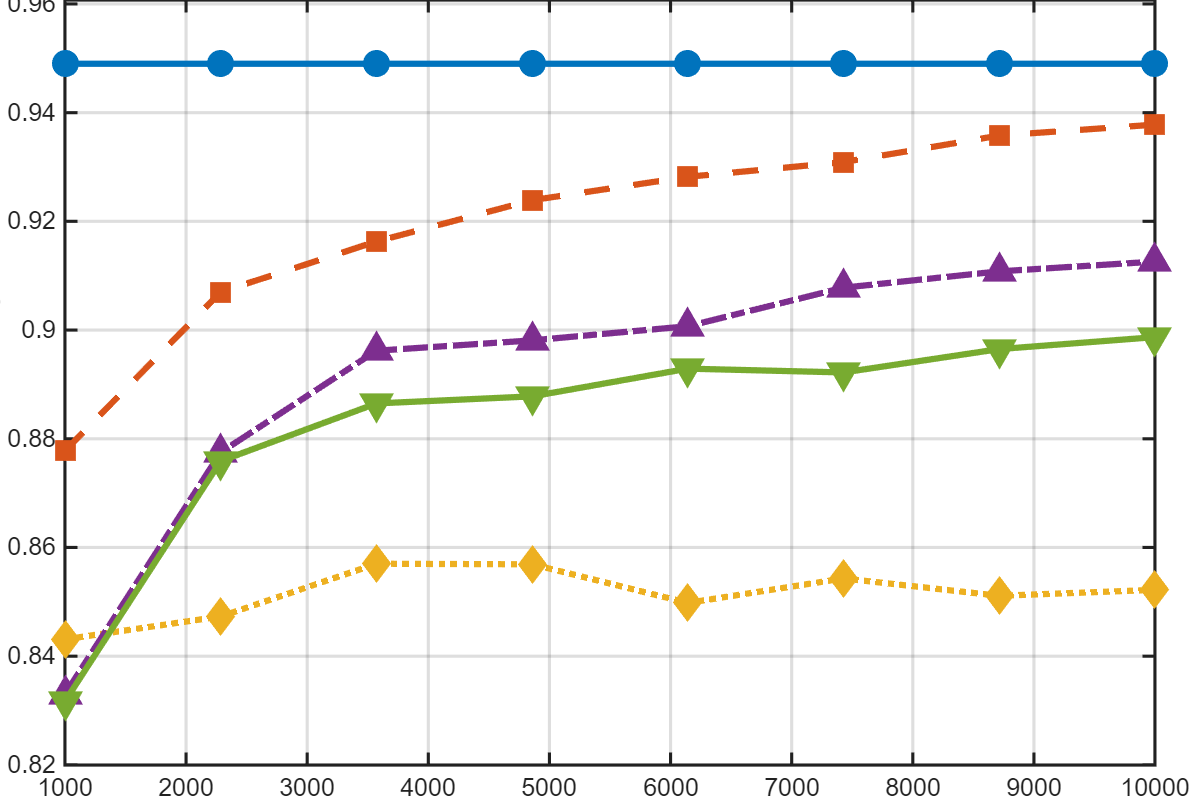}
		} \\
		\subfigure[Pathbased]{
			\includegraphics[width=0.1325\textwidth]{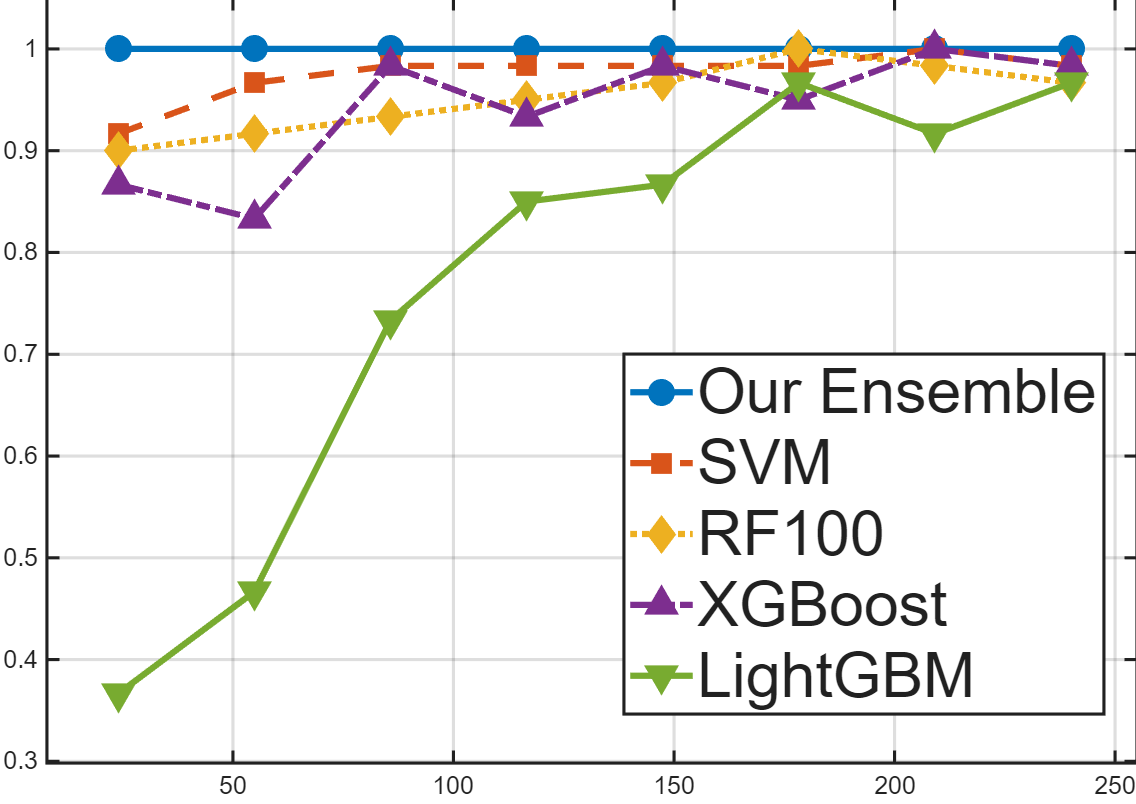}
		} &
		\subfigure[RELATHE]{
			\includegraphics[width=0.1325\textwidth]{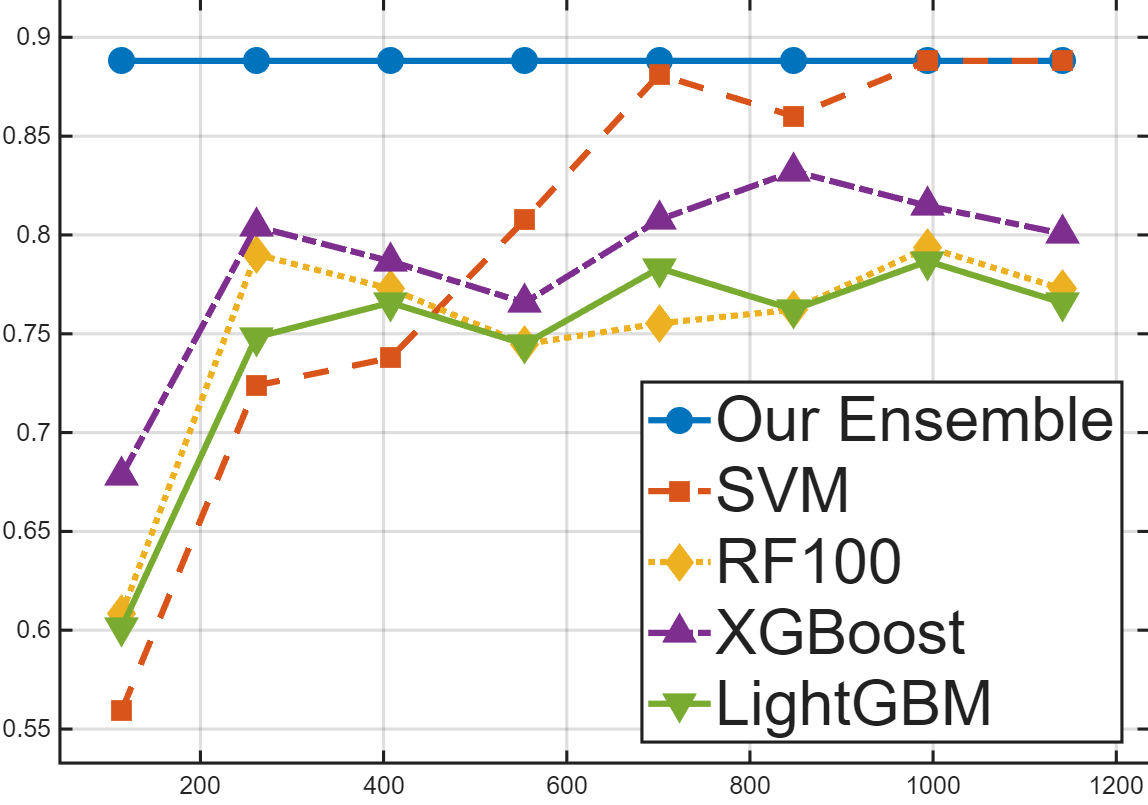}
		} &
		\subfigure[Wine]{
			\includegraphics[width=0.1325\textwidth]{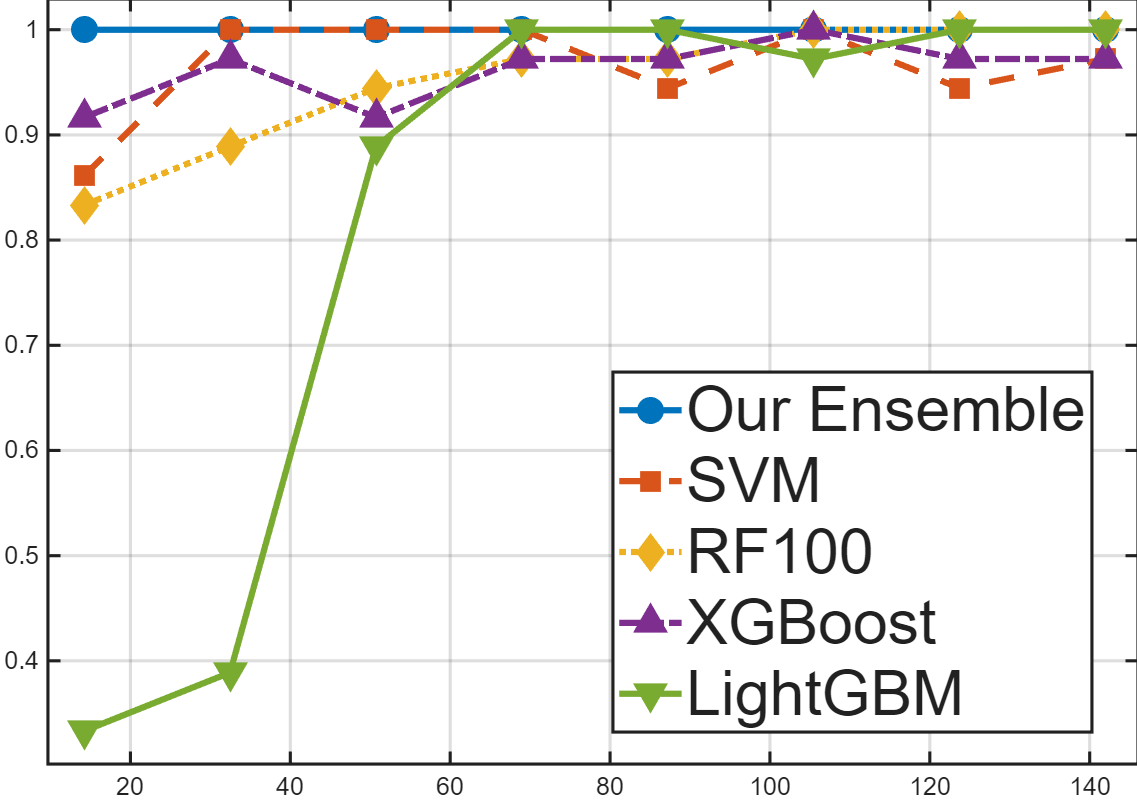}
		}
	\end{tabular}
	\caption{Accuracy vs. Training data size}
	\label{F6}
\end{figure}

\begin{small}
\begin{figure}[h]
	\centering
        \setlength{\tabcolsep}{2pt} 
	\begin{tabular}{cc}
		\subfigure[Test accuracy vs $\lambda$]{
			\includegraphics[width=0.22\textwidth]{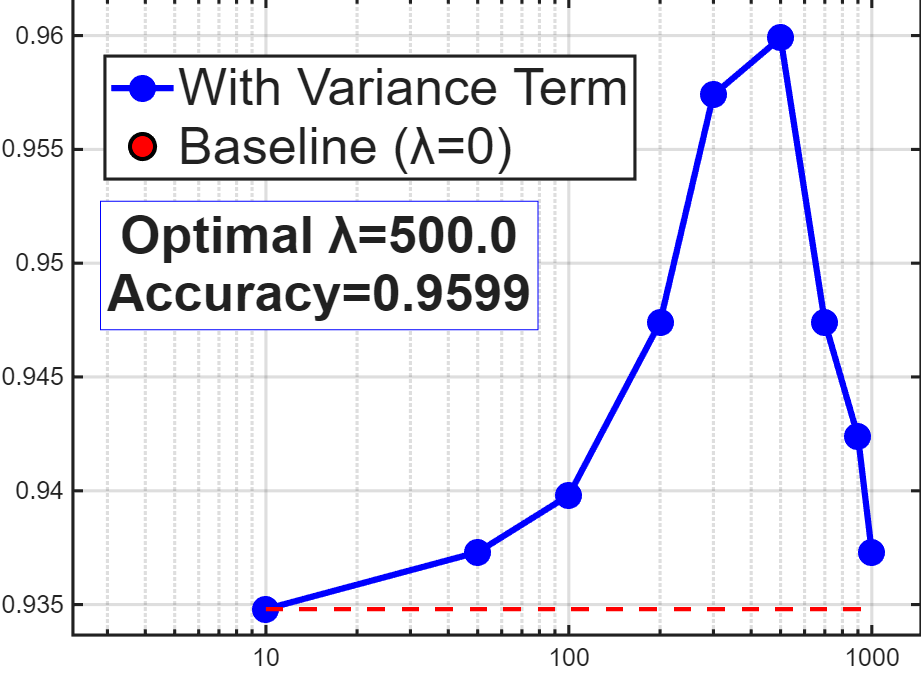} 
		} &
		\subfigure[Generalization gap vs $\lambda$]{
			\includegraphics[width=0.22\textwidth]{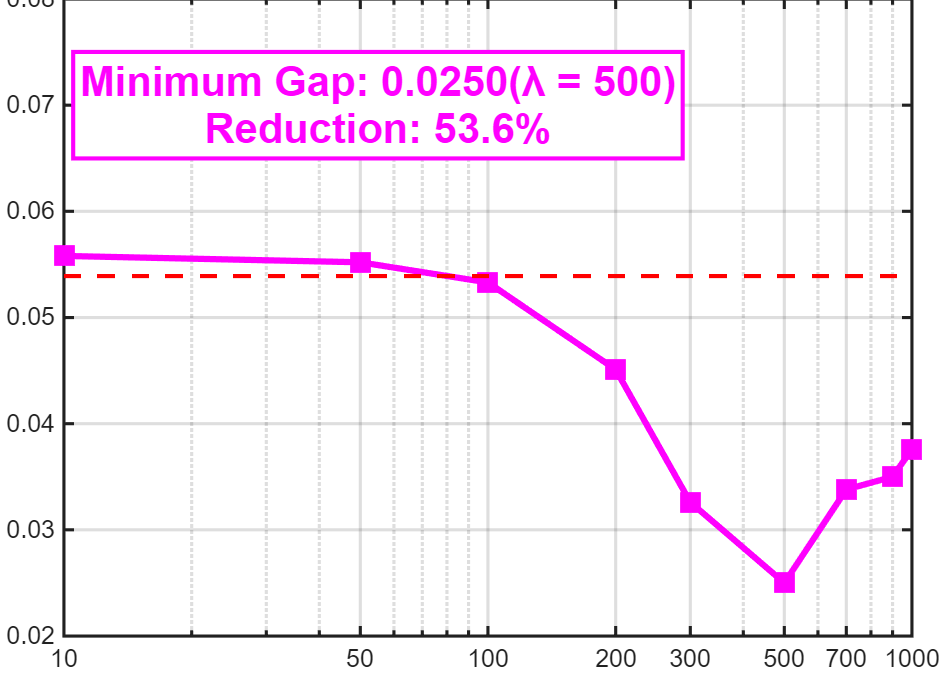} 
		} \\
		\subfigure[Training accuracy vs $\lambda$]{
			\includegraphics[width=0.22\textwidth]{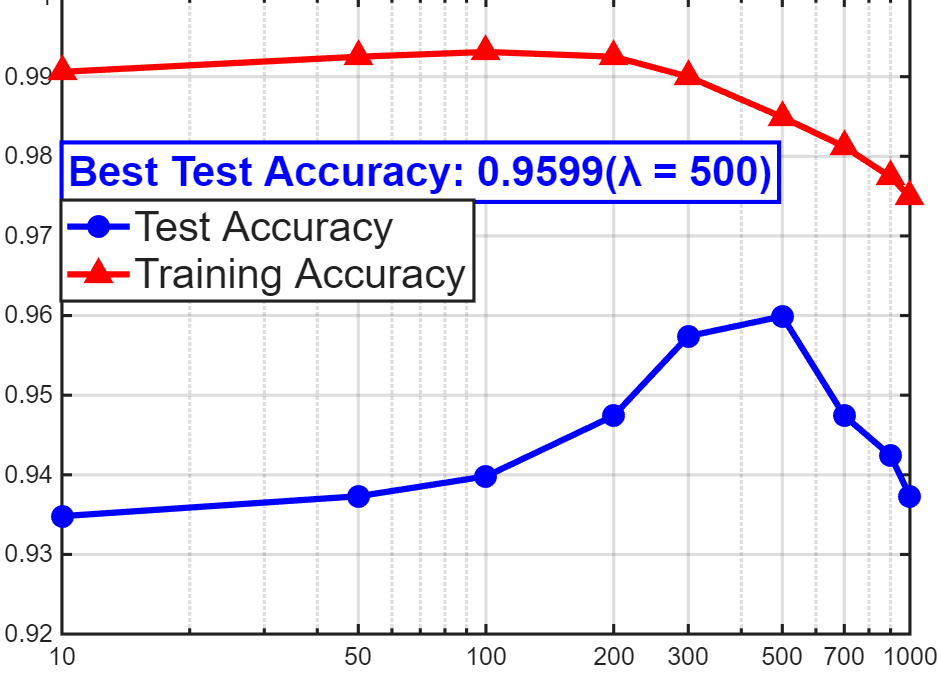} 
		} &
		\subfigure[Improvement vs $\lambda$]{
			\includegraphics[width=0.22\textwidth]{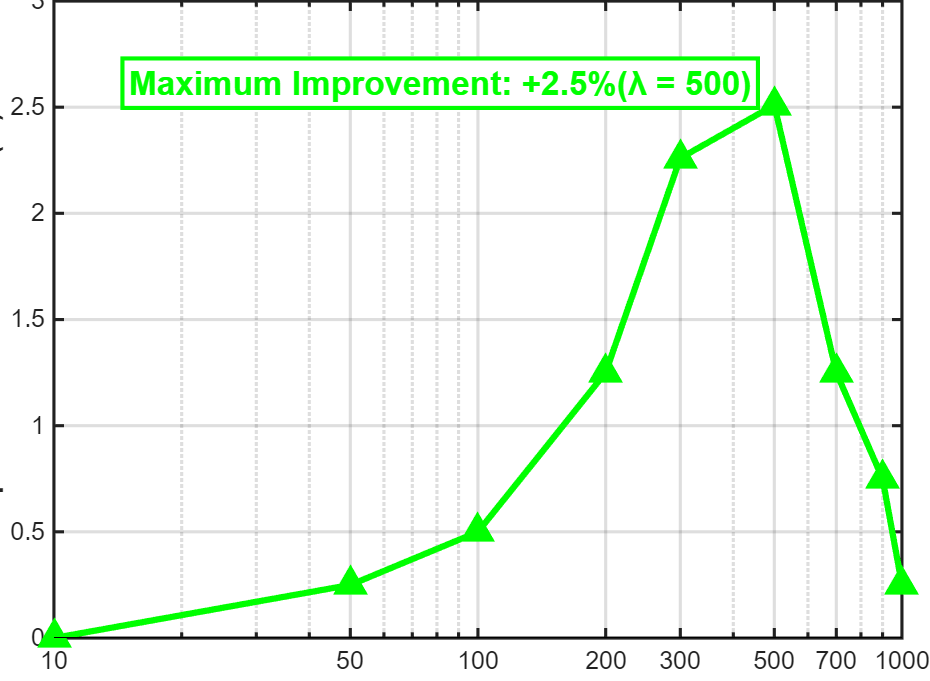}
		}
	\end{tabular}
	\caption{Sensitivity analysis of regularization parameter $\lambda$ on BASEHOCK dataset}
	\label{fig:lambda_sensitivity}
	\vspace{-\baselineskip} 
\end{figure}
\end{small}
Figure 2 demonstrates the learning curves of our ensemble method, indicating that our approach achieves competitive accuracy with remarkably limited training samples\cite{wooldridge_thirty-eighth_2024, wooldridge_thirty-eighth_2024-1}, highlighting its exceptional data efficiency and rapid convergence characteristics.

Figure 3 presents a comprehensive sensitivity analysis\cite{GAN2014269} of the regularization parameter $\lambda$ on the BASEHOCK dataset. The results reveal that an appropriately chosen $\lambda$ effectively enhances predictive accuracy while simultaneously reducing the generalization gap. However, excessively large $\lambda$ values may lead to underfitting\cite{montesinos_lopez_overfitting_2022}, manifesting as diminished performance on both training and test sets.

Overall, our method achieves superior performance with minimal computational resources, demonstrating strong generalization capability and robustness against overfitting across diverse datasets.
    \section{CONCLUSION}
\label{sec:conclusion}

In this paper, we propose a novel ensemble learning method that establishes a well-defined margin and incorporates both the mean and variance of the margin into the loss function. The approach employs Hadamard parameterization to simplify the optimization problem and leverages Riemannian gradients for effective solution.Experimental results demonstrate that our ensemble with only 10 learners exceeds the accuracy of 100-tree random forests, while also showing faster training and greater resilience to overfitting. We plan to investigate extensions to broader families of models and more challenging learning settings in subsequent research.

    \vfill\pagebreak
    \bibliographystyle{IEEEtran}
    \bibliography{ref}

\end{document}